\documentclass[runningheads,a4paper]{llncs}

\usepackage{amssymb}
\usepackage{amsmath}
\usepackage[all]{xy}
\usepackage{stmaryrd}
\usepackage{ascii}
\usepackage{enumerate}
\usepackage{booktabs} 
\usepackage{colortbl} 
\usepackage{hhline} 
\usepackage{graphicx} 
\usepackage{cancel}
\usepackage{color}

\usepackage{enumitem}

\usepackage[numbers]{natbib}
\makeatletter
\renewcommand\bibsection%
{
  \section*{\refname
    \@mkboth{\MakeUppercase{\refname}}{\MakeUppercase{\refname}}}
}
\makeatother

\usepackage{url}

\usepackage[page,title]{appendix}

\newtheorem{observation}{Observation}

\newcommand{\boxx}{\square}
\newcommand{\diaa}{\lozenge}
\newcommand{\Atoms}{{\bf P}}

\newcommand{\Lan}{\mathcal{L}}
\newcommand{\lequiv}{\leftrightarrow}
\newcommand{\limp}{\rightarrow}

\newcommand{\true}[1]{ \left\llbracket #1  \right\rrbracket  }

\newcommand{\IFF}{\Longleftrightarrow}

\newcommand{\betterthaneq}{\succcurlyeq}

\newcommand{\worsethaneq}{\preccurlyeq}

\newcommand{\Model}{\mathcal{M}}

\newcommand{\Val}{\mathcal{V}}

\renewcommand{\phi}{\varphi}

\newcommand{\putaway}[1]{}

\newcommand{\oblig}{{\bf O}}
\newcommand{\perm}{{\bf P}}

\newcommand{\Obdot}{\hat{\mathbf O}}

\newcommand{\set}[1]{\{ #1 \}} 

\newcommand{\Everywhere}{\boxx}
\newcommand{\Somewhere}{\diaa}

\newcommand{\Ob}[2]{{\mathbf O}(#1|#2)}

\begin{document}

\title{A Note on Nesting in Dyadic Deontic Logic}
\titlerunning{A Note on Nesting in Dyadic Deontic Logic}

\author{Agneau Belanyek\inst{1} \and Davide Grossi\inst{2} \and Wiebe van der Hoek\inst{3} \\
University of Liverpool}
\authorrunning{Belanyek, Grossi and van der Hoek}

\institute{\email{a.belanyek@liverpool.ac.uk}
			\and  \email{d.grossi@liverpool.ac.uk}
			\and \email{wiebe@csc.liv.ac.uk}	}

\maketitle

\begin{abstract}
The paper reports on some results concerning \AA qvist's dyadic logic known as {\em system G}, which is one of the most influential logics for reasoning with dyadic obligations (``it ought to be the case that \ldots if it is the case that \ldots''). Although this logic has been known in the literature for a while, many of its properties still await in-depth consideration. In this short paper we show: that any formula in system G including nested modal operators is equivalent to some formula with no nesting; that the universal modality introduced by \AA qvist in the first presentation of the system is definable in terms of the deontic modality.


\end{abstract}

\section{Introduction}

Modern research into the use of deontic logic to represent and reason with normative statements began with the introduction in \cite{vonWright51deontic} of what has become known as standard deontic logic (SDL). SDL is based on a propositional language and uses modal operators for obligation and permission where, intuitively, $\oblig\phi$ means that it ought to be the case (or, it is obligatory) that the formula $\phi$ be true, and as the dual, $\perm\phi$ means it is permitted that the formula $\phi$ be true. 

Dyadic deontic logic (DDL) was introduced in \cite{VonWright1964} to cater for paradoxes in SDL resulting from so-called contrary-to-duty obligations as described in \cite{Chisholm1963}. In DDL, dyadic modal operators (adapted from the conventions used to express conditional probabilities) enable modal formulae to express the conditions in which obligations and permissions apply so that, for example, $\oblig(\phi|\psi)$ means that it is obligatory that the formula $\phi$ be true, provided or given that $\psi$ is true.

The standard semantics used for DDL were introduced by \citeauthor{Hansson1969} in \cite{Hansson1969}, who imposes restrictions on the language such that modal operators cannot be nested and mixed formulae (connecting modal and propositional formulae) are not permitted. 

In \cite{Aqvist1986}, \citeauthor{Aqvist1986} proposes system G, a system of dyadic deontic logic which extends \citeauthor{Hansson1969}'s logic with a
universal necessity operator $\Everywhere$. Intuitively $\Everywhere \phi$  means that the formula $\phi$ is necessarily true (and the dual $\Somewhere \phi$ means the formula $\phi$ is possibly true).\footnote{\cite{Aqvist1986} uses the operators $N$(necessarily) and $M$(maybe), but we have opted for the standard `box' and `diamond' normally used with alethic logics and consistent with other recent authors such as Parent (for example, in \cite{DBLP:conf/deon/Parent08}).}



We present an interesting property of System G, namely that every formula with nested $\oblig$ and $\Everywhere$ operators is equivalent to another formula without any nesting. It has already been established that this is the case for $\mathsf{S5}$ logics in \cite{Meyer1995}. In System G, the $\Everywhere$ operator is an $\mathsf{S5}$ operator and it therefore follows from \cite{Meyer1995} that nestings of the $\Everywhere$ are superfluous. The $\oblig$ operator however presents additional challenges as the nesting may occur within either argument of the operator. We can still prove that despite this, nestings of the $\oblig$ operator and of combinations of the $\oblig$ and $\Everywhere$ operator are superfluous. These results provide an {\em a posteriori} justification to Hansson's original syntactic restrictions.

\medskip

The rest of the paper is structured as follows. In Section~\ref{sec:preliminaries} we introduce some preliminaries including the syntax, semantics and proof theory of System G. In Section~\ref{sec:nestedformulae} we show our main result that every System G formula with nested $\oblig$ and $\Everywhere$ operators is equivalent to another formula without any nesting. We do this using both a semantic and a syntactic argument. As part of the proof, we also show that in System G, all  $\Everywhere$-formulae can be expressed using  the $\oblig$ modality only. We conclude in In Section~\ref{sec:discussion}. Longer proofs of the lemmas and theorems are provided in the appendix.


\section{Preliminaries} \label{sec:preliminaries}

\subsection{Some background on nesting deontic operators}

The issue of whether deontic logic formalisms should allow operators to be nested within the scope of others has been considered as early as in \cite{Marcus1966}. 
Sentences in which deontic operators are nested are sentences like (taken from \cite{Marcus1966}):
\begin{itemize}
	\item It is obligatory that it is obligatory that everyone keeps his promises.\label{ex:OOp}
	\item It ought to be the case that what ought to be the case is the case. \label{ex:OOaTHENa}
\end{itemize}
The issue, from a deontic logic point of view, is whether statements such as the above ones should be objects in a logic or should rather be ignored (on different bases such as being non-meaningful, or trivial or uninteresting). 

As pointed out in \cite{Goble1966} (quoted in \cite{Wansing1998}), little attention was initially paid to nested modalities. In \cite{VonWright1964}, \citeauthor{VonWright1964} defines the language of both his original and dyadic systems to exclude nesting with no explicit discussion of why this should be the case. \citeauthor{Hansson1969} in \cite{Hansson1969} follows this tradition and also explicitly forbids nesting in his language, commenting however that despite this restriction ``almost all philosophical problems discussed in connection with deontic logic are expressible''. In \cite{Aqvist1986}, \citeauthor{Aqvist1986} identifies nesting as one of the differences between \citeauthor{Hansson1969}'s language and his proposed System G but does not discuss any implications of this difference. 

Some work has been done on the representation and meaning of nested modalities in some normative systems, such as in so-called \emph{stit} logics in \cite{Belnap1995}. However, despite the fact that it is considered a powerful logic for handling contrary-to-duty obligations, to the best of our knowledge, no research has been done on nested modalities in System G. In this paper we aim to address this and provide a technical answer to the issue of nesting (within System G).

Results of this type are well-known for normal monadic modal logics (e.g., for $\mathsf{S5}$ \cite{Meyer1995}). In System G, the $\Everywhere$ operator is an $\mathsf{S5}$ operator and it therefore follows from \cite{Meyer1995} that nestings of the $\Everywhere$ are superfluous. The $\oblig$ operator however presents additional challenges as the nesting may occur within either argument of the operator. We prove that despite this, nestings of the $\oblig$ operator and of combinations of the $\oblig$ and $\Everywhere$ operator are superfluous.

\subsection{Language, Syntax and Semantics of System G}\label{sec:systemG}

\subsubsection{Language}

Formally the language of system G is defined thus

\begin{definition}[Language of System G]\label{def:systemG}
Let $\Atoms$ be a set of propositional atoms. The language of system G, $\Lan$ is defined by the following inductive syntax, where $p \in \Atoms$ and $\phi \in \Lan$
\[
\Lan :  p \ |\ \neg \phi \ |\  \phi \lor \phi \ |\  \oblig(\phi|\phi) \ |\  \Everywhere \phi
\]

\end{definition}

Conjunction, material implication and biconditional are defined from negation and disjunction in the expected way. We treat $\top$ as a propositional atom in $\Atoms$, and define $\bot$ as $\neg \top$. Parentheses may be used to clarify the order of operators. 

$\oblig(\phi_1|\phi_2)$ is read as `$\phi_1$ is obligatory, given $\phi_2$'. $\Everywhere \phi$ is read as `Everywhere $\phi$'. It follows that the duals of these operators $\perm$(`Permitted') and $\Somewhere$(`Somewhere') are defined as follows:
\[
\perm(\phi_1|\phi_2) =_{\mathit{def}} \neg \oblig(\neg \phi_1|\phi_2)
\]
\[
\Somewhere \phi =_{\mathit{def}} \neg \Everywhere \neg \phi
\]

\subsubsection{Semantics}

These formulae are evaluated using Kripke-style possible world semantics based on the so-called class of $Strong\ H_3$ models defined by \citeauthor{Aqvist1986} in \cite{Aqvist1986} as follows:

\begin{definition}[$Strong\ H_3$ model]\label{def:strongh3model}
Given a set of propositional atoms $\Atoms$, a $Strong\ H_3$ model $\Model$ is a structure $\langle S, \worsethaneq, \Val \rangle$, where 

\begin{itemize}
	\item $S$ is a non-empty set of possible worlds
	
	\item $\worsethaneq$ is a reflexive, transitive, fully connected and limited (see below) binary relation over $S$ that satisfies the \emph{limit assumption} condition described below.
	
	\item $\Val$ is a valuation function  $\Val: \Atoms \to \wp(S)$ that assigns a truth set to every proposition in $\Atoms$.
\end{itemize} 
\end{definition}

The weak preference relation $\worsethaneq$ captures the idea that for any two worlds $t$ and $s$, if $t \worsethaneq s$ then $s$ is at least as good as $t$. 
The relation is \emph{fully connected} if and only if any two worlds in $S$ are comparable, or formally that:
\begin{eqnarray}
\forall s,t \in S, s \worsethaneq t \mbox{ or } t \worsethaneq u
\end{eqnarray}

The limitedness property can be characterised in terms of the $opt$ function, defined as follows.
\begin{definition}[$opt$ function]
\begin{eqnarray}
opt(X) & = &  \set{x \in X \mid \forall y \in X: y \worsethaneq x}\label{eq:optdef}
\end {eqnarray}
\end{definition}

Intuitively, $opt(X)$ are the elements of $X$ that are at least as good as any other element in $X$. We say that $opt(X)$ is the \emph{optimal set} or the set of \emph{optimal elements} of X.

The preference relation $\betterthaneq$ is then \emph{limited} if and only if for every non-empty subset of $S$ there is at least one optimal element, or formally:
\begin{eqnarray}
\forall X \subseteq S, \mbox{if } X \neq \emptyset \mbox{ then } opt(X) \neq \emptyset \label{eq:limass}
\end{eqnarray}

\begin{definition}[Truth at a point]
The notion of a formula $\phi \in \Lan$  being true at a world s in a model $\Model = \langle S, \worsethaneq, \Val \rangle$, denoted $\Model, s \models \phi$, is defined inductively as follows:
\begin{align}
  \Model,s \models  \top & \\
	\Model,s \models p &\IFF s \in \Val(p) \\
	\Model,s \models \neg \phi &\IFF \mbox{not }\Model,s \models \phi \\
	\Model,s \models \phi_1 \lor \phi_2 &\IFF  \Model,s \models \phi_1 \mbox{ or } \Model,s \models \phi_2 \\
  \Model,s \models \oblig(\phi_1 \mid \phi_2)  &\IFF   opt(\true{\phi_2}) \subseteq \true{\phi_1} \label{eq:ObligDef} \\ 
	\Model,s \models \Everywhere \phi &\IFF  \forall (t \in S) \Model,t \models \phi\label{eq:EverywhereDef}
\end{align}
\end{definition}

($\true{\phi}$ denotes the \emph{truth set} of $\phi$, i.e. the set of worlds in which $\phi$ is true.)\\

The truth definitions for $\perm$ (the dual of $\oblig$) and $\Somewhere$(the dual of $\Everywhere$), derived from \eqref{eq:ObligDef} and \eqref{eq:EverywhereDef} are as follows:
\begin{align}
 \Model,s \models \perm(\phi_1 \mid \phi_2)  &\IFF   opt(\true{\phi_2}) \cap \true{\phi_1} \neq \emptyset \label{eq:PermDef}\\
 \Model,s \models \Somewhere \phi  &\IFF  \exists (t \in S) \Model,t  \models \phi\label{eq:SomewhereDef}
\end{align} 

We say that a formula $\phi$ is \emph{valid} in $Strong\ H_3$ models if it is true in all worlds of all $Strong\ H_3$ models, and denote this $\models_{H_3} \phi$. (We sometimes drop the subscript if the class of models being referred to is clear.) If $\phi$ is valid with respect to a class of models, we say that $\phi$ is a \emph{validity} of that class of models. We say that two formulae $\phi$ and $\psi$ are semantically equivalent if $\phi \lequiv \psi$ is a validity.

\begin{observation}\label{obs:globalmodalprops}
The modalities $\oblig$ and $\Everywhere$, as well as their duals, are global modalities in the following sense:

\begin{align*}
&\exists s \in S : \Model,s \models \oblig(\phi \mid \psi)  \mbox{ iff }  \forall t \in S, \Model,t \models \oblig(\phi \mid \psi) \\
&\exists s \in S : \Model,s \models \Everywhere \phi  \mbox{ iff }  \forall t \in S, \Model,t \models \Everywhere \phi
\end{align*}
\end{observation}

Note that, as a consequence, for a formula of the form $\oblig(\phi|\psi)$ in any model $\Model =  \langle S, \worsethaneq, \Val \rangle$, we either have $\true{\oblig(\phi|\psi)}=S$ or $\true{\oblig(\phi|\psi)}= \emptyset$. Likewise for the other modalities.

\subsubsection{Axiomatics}

\AA qvist proposes the following axioms and rules for system G (names are from \cite{DBLP:conf/deon/Parent08}).
\begin{align}
	&\mbox{All propositional tautologies} \tag{\tt PL}\\
	&\mbox {$\mathsf{S5}$ axioms for }\Everywhere \tag{\tt S5}\\
	&\perm(\phi|\psi) \lequiv \neg \oblig(\neg \phi|\psi) \tag{\tt DfP}\\
	&\oblig(\phi_1 \limp \phi_2|\psi) \limp (\oblig(\phi_1|\psi) \limp \oblig(\phi_2|\psi)) \tag{\tt COK}\\
	&\oblig(\phi|\psi) \limp \Everywhere \oblig(\phi|\psi) \tag{\tt Abs}\\
	&\Everywhere\phi \limp \oblig(\phi|\psi) \tag{\tt CON}\\
	&\Everywhere(\psi_1 \lequiv \psi_2) \limp ( \oblig(\phi|\psi_1) \lequiv \oblig(\phi|\psi_2) ) \tag{\tt Ext}\\
	&\oblig(\phi|\phi) \tag{\tt Id}\\
	&\oblig(\phi|\psi \land \chi) \limp \oblig(\chi \limp \phi|\psi) \tag{\tt C}\\
	&\Somewhere(\psi) \limp (\oblig(\phi|\psi) \limp \perm(\phi|\psi)) \tag{\tt D*}\\
	&\perm(\phi|\psi) \land  \oblig(\phi \limp \chi | \psi)  \limp \oblig(\chi | \phi \land \psi)  \tag{\tt S}\\ \nonumber \\
	&\mbox{If }\vdash \phi \mbox{ and } \vdash \phi \limp \psi \mbox{ then } \vdash \psi \tag{\tt MP}\\
	&\mbox{If } \vdash \phi \mbox{ then } \vdash \Everywhere \phi \tag{\tt N} 	
\end{align}

We refer to this axiom system as system G, and so we say that a formula $\phi$ is \emph{derivable in system G} if it can be derived using this calculus. We denote this by $\vdash_G \phi$. (For readability, we sometimes drop the subscript if the system referred to is clear).  If $\phi$ is derivable in system G, we say that $\phi$ is a \emph{theorem} of system G. We say that two formulae $\phi$ and $\psi$ are \emph{provably equivalent} if $\phi \lequiv \psi$ is a theorem of System G.

\citeauthor{DBLP:conf/deon/Parent08} \cite{DBLP:conf/deon/Parent08} shows that this axiom system is sound and strongly complete with respect to $Strong\ H_3$ models.

\section{Nested formulae}\label{sec:nestedformulae}

In this section we present our main finding that any system G formula which contains nested modal operators is equivalent to some other formula without nesting. We present a semantic and a syntactic argument for this theorem.

Both arguments use the following definitions and lemma.

\begin{definition} [Modal depth]
The modal depth $md(\phi)$ of a formula $\phi \in \Lan$ is defined inductively as 
\begin{align*}
   md(\top) &= 0 \\
	 md(p)    &= 0, \forall p \in \Atoms \\
	 md(\neg \phi) &= md(\phi)\\
	 md(\phi_1 \lor \phi_2) &= \max(md(\phi_1),md(\phi_2))\\
	 md(O(\phi_1|\phi_2)) &= 1 + \max(md(\phi_1),md(\phi_2))\\
	 md(\Everywhere\phi) &= 1 + md(\phi)
\end{align*}

where $\max$ represents the arithmetic maximum, i.e. for any $x,y \in \mathbb{N}$, if $x \geq y$ then $\max(x,y) = x$, otherwise $\max(x,y) = y$.

\end{definition}

\begin{definition}[Unnested disjunctive normal form (UDNF)]\label{def:udnf}
Let $\Lan^\oblig$ be the sublanguage of $\Lan$ without formulae containing the $\Everywhere$-operator.
We say that a formula $\psi \in \Lan^\oblig$ is in {\em Unnested Disjunctive Normal Form (UDNF)} if it is a disjunction of conjunctions of the form
\[
\delta = \alpha \land \oblig(\varphi_1 \mid \psi_1) \land \dots \land  \oblig(\varphi_n \mid \psi_n) \land \neg\oblig(\varphi_{n+1} \mid \psi_{n+1}) \land \dots \land \neg\oblig(\varphi_{n+k} \mid \psi_{n+k})
\]
where $n, k \in \mathbb{N}$ and all of the formulae $\alpha, \varphi_m, \psi_m$ ($m \leq n + k$) are propositional formulae ($\top$ and $\bot$ are considered propositional formulae). The formula $\delta$ is called a {\em canonical conjunction} and the formulae $\oblig(\varphi_m \mid \psi_m)$ are called {\em prenex} formulae.
\end{definition}

Observe that by the above definition, formulae in UDNF have a maximum modal depth of 1.

The following lemma from \cite{Meyer1995} guarantees that a prenex formula within a formula in UDNF can always be moved to the outermost level.

\begin{lemma}\label{lem:prenex}(\citeauthor{Meyer1995} \cite{Meyer1995} Lemma 1.7.6.2) If $\psi$ is in UDNF and contains a prenex formula $\sigma$, then $\psi$ is equivalent to a formula of the form $\pi \lor ( \lambda \land \sigma)$ where $\pi$, $\lambda$ and $\sigma$ are all in UDNF.
\end{lemma}

\begin{proof}
$\psi$ is in UDNF so $\psi = \delta_1 \lor \delta_2 \lor \dots \lor \delta_m$ where all the $\delta_i$'s are canonical conjunctions. Suppose $\sigma$ occurs in $\delta_m$. As $\sigma$ is one of the conjuncts of $\delta_m$, $\delta_m$ can be written as $\lambda \land \sigma$, where $\lambda$ collects the remaining conjuncts in $\delta_m$ (or $\top$ if $\delta_m = \sigma$).  Taking $\pi$ to be $(\delta_1 \lor \delta_2 \lor \dots \lor \delta_{m-1})$ gives the desired result $\psi = \pi \lor (\lambda \land \sigma)$.
\end{proof}
  
\subsection{Semantic argument}

First we state the following proposition which provides a means to obtain for every $\Everywhere$-formula, a semantically equivalent $\oblig$-formula.

\begin{proposition}\label{prop:Def}


\begin{align}
\models \Everywhere \phi & \lequiv \oblig(\bot|\neg \phi)
\end{align}

\end{proposition}




The following lemma guarantees that an $\oblig$ or $\neg \oblig$ formula within the scope of another can be brought out of that scope.
\begin{lemma} \label{lem:equiv} Let $\Obdot$ stand for an arbitrary but fixed $\Ob{\varphi'}{\psi'}$ formula.


\begin{align}
\label{item:equiv:one} &\models \oblig(\varphi | (\pi \lor (\lambda \land \Obdot))) \lequiv ((\Obdot \land \oblig(\varphi|\pi \lor \lambda)) \lor (\neg\Obdot \land \oblig(\varphi|\pi))\\
\label{item:equiv:two} &\models \oblig(\varphi | (\pi \lor (\lambda \land \neg \Obdot))) \lequiv ((\neg \Obdot \land \oblig(\varphi|\pi \lor \lambda)) \lor (\Obdot \land \oblig(\varphi|\pi))\\
\label{item:equiv:three} &\models \oblig(\pi \lor (\lambda \land \Obdot)|\psi) \lequiv ((\Obdot \land \oblig(\pi \lor \lambda|\psi)) \lor (\neg \Obdot \land \oblig(\pi|\psi)))\\
\label{item:equiv:four} &\models \oblig(\pi \lor (\lambda \land \neg \Obdot)|\psi) \lequiv ((\neg \Obdot \land \oblig(\pi \lor \lambda|\psi)) \lor ( \Obdot \land \oblig(\pi|\psi)))
\end{align}

\end{lemma}


We can now state our main theorem as follows.

\begin{theorem}\label{thm:nonestings} 
For every formula $\chi \in \Lan$, there exists a formula $\chi'$ such that $\chi'$ is in UDNF and $\models \chi \lequiv \chi'$
\end{theorem}


\subsection{Syntactic argument}

The syntactic argument has the same structure as the semantic argument with Proposition~\ref{prop:Def} and Lemma~\ref{lem:equiv} replaced by syntactic counterparts. 

The following is a syntactic counterpart to Proposition~\ref{prop:Def}.
\begin{proposition}\label{prop:SynDef}
\begin{align}
  \vdash \Everywhere  \phi & \lequiv \oblig(\bot|\neg \phi) \label{eqn:SynEverywheretoO}
\end{align}

\end{proposition}


%

The following is a syntactic counterpart to Lemma~\ref{lem:equiv}.
\begin{lemma}\label{lem:Synequiv}Let $\Obdot$ stand for an arbitrary but fixed $\Ob{\varphi'}{\psi'}$ formula. 
\begin{align}
\label{item:equiv:five} &\vdash \oblig(\varphi | (\pi \lor (\lambda \land \Obdot))) \lequiv ((\Obdot \land \oblig(\varphi|\pi \lor \lambda)) \lor (\neg\Obdot \land \oblig(\varphi|\pi))\\
\label{item:equiv:six} &\vdash \oblig(\varphi | (\pi \lor (\lambda \land \neg \Obdot))) \lequiv ((\neg \Obdot \land \oblig(\varphi|\pi \lor \lambda)) \lor (\Obdot \land \oblig(\varphi|\pi))\\
\label{item:equiv:seven} &\vdash \oblig(\pi \lor (\lambda \land \Obdot)|\psi) \lequiv ((\Obdot \land \oblig(\pi \lor \lambda|\psi)) \lor (\neg \Obdot \land \oblig(\pi|\psi)))\\
\label{item:equiv:eight} &\vdash \oblig(\pi \lor (\lambda \land \neg \Obdot)|\psi) \lequiv ((\neg \Obdot \land \oblig(\pi \lor \lambda|\psi)) \lor ( \Obdot \land \oblig(\pi|\psi)))
\end{align}
\end{lemma}


\begin{theorem}\label{thm:Synnonestings}
Every formula $\chi$ is provably equivalent to a formula in UDNF. 
\end{theorem}
\begin{proof}
The proof is by induction on $\chi$ using the same argument as the proof for Theorem~\ref{thm:nonestings} except that Proposition~\ref{prop:Def} and Lemma~\ref{lem:equiv} are replaced by their syntactic counterparts Proposition~\ref{prop:SynDef} and Lemma~\ref{lem:Synequiv} respectively.
\end{proof}


\section{Conclusion}\label{sec:discussion}

When he introduced restrictions on nesting and in \cite{Hansson1969},\citeauthor{Hansson1969} commented that in SDL, despite these restrictions, ``almost all philosophical problems discussed in connection with deontic logic are expressible in his language.'' He, however, offers no justification for this conjecture and, in moving from SDL to the dyadic language, a formal definition of the dyadic language used in not given, and it is not explicit whether these restrictions are to be kept. It is generally considered that it was his intention that the dyadic language be so restricted, and this is stated explicitly, for example, in \cite{Spohn1975}. We note also that the list of example valid and invalid formulae given in \cite{Hansson1969} contain no nested formulae.

From a technical point of view, it seems unnatural to forbid certain operators to appear in the scope of others. The main goal of our work has been to see whether such a restriction is necessary, and whether it actually limits the ability of the system to express deontic concepts. We have demonstrated that the answer to both questions is negative: forbidding nested modalities is not technically needed, and, on the other hand, not allowing them does not restrict the expressive power of the language either. This  in some sense settles Hansson's conjecture, at least with respect to system G: if any interesting problems can be expressed in the full dyadic language, they also can in a restricted version where nesting is not permitted.




\newpage
\appendix
    \begin{center}
      {\bf APPENDIX}
    \end{center}
\section{Proof of Proposition \ref{prop:Def}}\label{app:proofDef}
\begin{proof}

\textsc{From left to right}:

Assume that for an arbitrary $strong\ h_3$ model $\Model$ and world $s$, $\Model,s \models \Everywhere \phi$. 

Therefore, by the truth definition of $\Everywhere$ \eqref{eq:EverywhereDef}, $\forall (t \in S) \Model,t \models \phi$. 

Therefore, there are no $\neg \phi$-worlds, so $\true{\neg \phi} = \emptyset$. 

Therefore, by the definition of opt \eqref{eq:optdef} $opt(\true{\neg \phi}) = \emptyset \subseteq \true{\bot}$. 

Therefore , by the truth definition of $\oblig$ \eqref{eq:ObligDef}, $\Model, s \models \oblig(\bot|\neg \phi)$.

\textsc{From right to left}:

Assume $\Model, s \models \oblig(\bot|\neg \phi)$. 

Therefore, by the truth definition of $\oblig$ \eqref{eq:ObligDef}, $opt(\true{\neg \phi}) \subseteq \true{\bot}$. 

Therefore, as $\true{\bot} = \emptyset$,  $opt(\true{\neg \phi}) = \emptyset$. 

Therefore, given the limit assumption \eqref{eq:limass}, $\true{\neg \phi} = \emptyset$. 

Therefore $\forall (t \in S) \Model,t \models \phi$. 

Therefore, by the truth definition of $\Everywhere$ \eqref{eq:EverywhereDef}, $\Model, s \models \Everywhere \phi$.




\end{proof}

\section{Proof of Lemma \ref{lem:equiv}}\label{app:proofequiv}
\begin{proof} 


\emph{To prove \eqref{item:equiv:one}:\\} 
(Note that \eqref{item:equiv:two} can be proved using a similar argument, replacing $\Obdot$ with $\neg \Obdot$.)\\

\textsc{From left to right:}

\begin{itemize}                                          
\item[] Let $\Model = \langle S,\worsethaneq,\Val \rangle$ be a model, and let $s \in S$.

\item[] Assume $\Model,s \models \oblig(\varphi | (\pi \lor (\lambda \land \Obdot)))$

\item[] Therefore, by \eqref{eq:ObligDef}, $opt(\true{\pi \lor (\lambda \land \Obdot)}) \subseteq \true{\varphi}$. Call this (A). 

\item[]  Given Observation~\ref{obs:globalmodalprops}, there are two cases: Either $\forall t \in S, \Model, t \models \Obdot$  or $\forall t \in S, \Model, t \models \neg \Obdot$.

\item[] If $\forall t \in S,\  \Model, t \models \Obdot$
     \begin{itemize}
					\item[] Then $\true{\pi \lor (\lambda \land \Obdot)}=\true{\pi} \cup (\true{\lambda} \cap S)=\true{\pi} \cup \true{\lambda} = \true{\pi \lor \lambda}$
					\item[] Therefore, given (A), $opt(\true{\pi \lor \lambda}) \subseteq \true{\varphi}$
					\item[] Therefore, by equation~\eqref{eq:ObligDef}, $\Model, s \models \oblig(\varphi | \pi \lor \lambda)$
					\item[] Therefore, given that $\Model, s \models \Obdot$, $\Model, s \models (\Obdot \land \oblig(\varphi | \pi \lor \lambda)) $
					\item[] Therefore, $\Model, s \models (\Obdot \land \oblig(\varphi|\pi \lor \lambda)) \lor (\neg\Obdot \land \oblig(\varphi|\pi))$
		  \end{itemize}
  
\item[] If $\forall t \in S,\ \Model,t \models \neg \Obdot$:	
			\begin{itemize}
					\item[] Then $\true{\pi \lor (\lambda \land \Obdot)}=\true{\pi} \cup (\true{\lambda} \cap \emptyset)=\true{\pi}$
					\item[] Therefore, given (A), $opt(\true{\pi}) \subseteq \true{\varphi}$
					\item[] Therefore, by equation~\eqref{eq:ObligDef}, $\Model, s \models \oblig(\varphi|\pi)$
					\item[] Therefore, given that $\Model, s \models \neg \Obdot$, $\Model,s \models (\neg\Obdot \land \oblig(\varphi|\pi)$
					\item[] Therefore, $\Model, s \models (\Obdot \land \oblig(\varphi|\pi \lor \lambda)) \lor (\neg\Obdot \land \oblig(\varphi|\pi))$
					
			 \end{itemize}
\end{itemize}

\textsc{From right to left:}

\begin{itemize}
\item[]  Let $\Model = \langle S,\worsethaneq,\Val \rangle$ be a model, and let $s \in S$

item[] Assume $\Model,s \models \oblig(\varphi | \pi \lor \lambda) \land \Obdot ) \lor  ( \oblig(\varphi|\pi) \land \neg \Obdot)$.
\item[] There are two cases. Either $\Model,s \models \oblig(\varphi | \pi \lor \lambda) \land \Obdot $  or $\Model,s \models \oblig(\varphi|\pi) \land \neg \Obdot$.
\item[] If $\Model,s \models \oblig(\varphi | \pi \lor \lambda) \land \Obdot $ holds:
         \begin{itemize}
				     \item[] Then $\Model,s \models \oblig(\varphi | \pi \lor \lambda)$ and $\Model, s, \models \Obdot$
						 \item[] Given $\Model,s \models \oblig(\varphi | \pi \lor \lambda)$, by equation~\eqref{eq:ObligDef} $opt(\true{\pi \lor \lambda}) \subseteq \true{\varphi}$. Call this (B).
						 \item[] By observation~\ref{obs:globalmodalprops}, if $\Model, s, \models \Obdot$ then $\forall t \in S, \Model, t \models \Obdot$, therefore $\true{\Obdot}=S$.
						 \item[] Therefore $\true{\lambda} = \true{\lambda} \cap \true{\Obdot} = \true{\lambda \land \Obdot}$
						 \item[] Therefore $\true{\pi \lor \lambda} = \true{\pi \lor (\lambda \land \Obdot)}$
						 \item[] Therefore, we can substitute $\true{\pi \lor (\lambda \land \Obdot)}$ for $\true{\pi \lor \lambda}$ in (B), getting $opt(\true{\pi \lor (\lambda \land \Obdot)}) \subseteq \true{\varphi}$
						 \item[] Therefore, by \eqref{eq:ObligDef}, $\Model, s \models \oblig(\varphi | (\pi \lor (\lambda \land \Obdot)))$
					\end{itemize}
\item[] If $\Model, s \models \oblig(\varphi|\pi) \land \neg \Obdot$ holds:
					\begin{itemize}
					   \item[]  Then $\Model,s \models \oblig(\varphi | \pi)$ and $\Model, s, \models \neg \Obdot$
						 \item[] Given $\Model,s \models \oblig(\varphi | \pi)$, by equation~\eqref{eq:ObligDef} $opt(\true{\pi}) \subseteq \true{\varphi}$. Call this (C).
						 \item[] By observation~\ref{obs:globalmodalprops}, if $\Model, s, \models \neg \Obdot$ then $\forall t \in S, \Model, t \models \neg \Obdot$, therefore $\true{\Obdot}=\emptyset$.
						 \item[] Therefore $\true{\lambda \land \Obdot} = \true{\lambda} \cap \true{\Obdot} = \emptyset$
						 \item[] Therefore $\true{\pi} = \true{\pi \lor (\lambda \land \Obdot)}$
						 \item[] Therefore, we can subtitute $\true{\pi \lor (\lambda \land \Obdot)}$  for $\true{\pi}$ in (C), getting $opt(\true{\pi \lor (\lambda \land \Obdot)}) \subseteq \true{\varphi}$
						 \item[] Therefore, by \eqref{eq:ObligDef}, $\Model, s \models \oblig(\varphi | (\pi \lor (\lambda \land \Obdot)))$
					\end{itemize}
\end{itemize}

\emph{To prove \eqref{item:equiv:three}:\\} 
(Note that \eqref{item:equiv:four} can be proved using a similar argument, replacing $\Obdot$ with $\neg \Obdot$.)\\

\textsc{From left to right:}

\begin{itemize}
\item[] We assume that for some arbitrary model $\Model = \langle S,\worsethaneq,\Val \rangle$ and arbitrary world $s \in S$, $\Model,s \models \oblig(\ (\pi \lor (\lambda \land \Obdot))|\varphi \ )$

\item[] Therefore, by equation~\eqref{eq:ObligDef}, $opt(\true{\varphi}) \subseteq \true{\pi \lor (\lambda \land \Obdot)}$. Call this (D).


\item[] By Observation~\ref{obs:globalmodalprops}, there are two cases. Either $\forall t \in S,\ \Model,t \models \Obdot$ or $\forall t \in S,\ \Model,t \models \neg \Obdot$.

\item[] If $\forall t \in S,\ \Model,t \models \Obdot$:
			\begin{itemize}[]
					\item[] Then $\true{\Obdot} = S$
					\item[] Therefore, $\true{\pi \lor (\lambda \land \Obdot)}=\true{\pi} \cup (\true{\lambda} \cap S)=\true{\pi} \cup \true{\lambda} = \true{\pi \lor \lambda}$
					\item[] Therefore, given (D), $opt(\true{\varphi}) \subseteq \true{\pi \lor \lambda}$
					\item[] Therefore, by equation~\eqref{eq:ObligDef}, $\Model, s \models \oblig(\ (\pi \lor \lambda)|\varphi \ )$
					\item[] Therefore, given that $\Model, s \models \Obdot$, $\Model, s \models (\Obdot \land \oblig(\varphi|\pi \lor \lambda))$
					\item[] Therefore, $\Model, s \models (\Obdot \land \oblig(\varphi|\pi \lor \lambda)) \lor (\neg\Obdot \land \oblig(\varphi|\pi)$
			\end{itemize}
					
\item[] If $\forall t \in S,\ \Model,t \models \neg \Obdot$:	
			\begin{itemize}
					\item[] Then $\true{\Obdot} = \emptyset$
					\item[] Therefore, $\true{\pi \lor (\lambda \land \Obdot)}=\true{\pi} \cup (\true{\lambda} \cap \emptyset)=\true{\pi}$
					\item[] Therefore, given (D), $opt(\true{\varphi}) \subseteq \true{\pi}$
					\item[] Therefore, by equation~\eqref{eq:ObligDef}, $\Model, s \models \oblig(\pi|\varphi)$
					\item[] Therefore, given that $\Model, s \models \neg \Obdot$, $\Model, s \models (\neg\Obdot \land \oblig(\varphi|\pi)$
					\item[] Therefore, $\Model, s \models (\Obdot \land \oblig(\varphi|\pi \lor \lambda)) \lor (\neg\Obdot \land \oblig(\varphi|\pi)$
			\end{itemize}
\end{itemize}



\textsc{From right to left:}

\begin{itemize}
\item[] We assume that for some arbitrary model $\Model = \langle S,\worsethaneq,\Val \rangle$ and arbitrary world $s \in S$, $\Model,s \models ( \oblig(\pi \lor \lambda | \varphi) \land \Obdot) \lor (\oblig(\pi|\varphi) \land \neg \Obdot)$.

\item[] There are two cases. Either $\Model,s \models  \oblig(\pi \lor \lambda | \varphi) \land \Obdot$ or $\Model,s \models \oblig(\pi|\varphi) \land \neg \Obdot$

\item[] If $\Model,s \models  \oblig(\pi \lor \lambda | \varphi) \land \Obdot$:
		\begin{itemize}
		    \item[] By Observation~\ref{obs:globalmodalprops}, $\forall t \in S, \Model, t \models \oblig(\pi \lor \lambda | \varphi) \mbox{ and } \Model, t \models \Obdot$.
				\item[] Therefore, by equation~\eqref{eq:ObligDef}, $opt(\true{\varphi}) \subseteq \true{\pi \lor \lambda}$.
				\item[] Therefore, given $\true{\Obdot} = S $  , $opt(\true{\varphi}) \subseteq \true{\pi \lor (\lambda \land \Obdot)}$
				\item[] Therefore, by equation~\eqref{eq:ObligDef}, $\Model, s \models \oblig(\ (\pi \lor (\lambda \land \Obdot))|\varphi \ )$
		 \end{itemize}
\item[] If $\Model,s \models  \oblig(\pi|\varphi) \land \neg \Obdot$:
		\begin{itemize}
		    \item[] By Observation~\ref{obs:globalmodalprops}, $\forall t \in S, \Model, t \models \oblig(\pi | \varphi) \mbox{ (and } \Model, t \models \neg \Obdot$ ).
				\item[] Therefore, by equation~\eqref{eq:ObligDef}, $opt(\true{\varphi}) \subseteq \true{\pi}$.
				\item[] Therefore, given that $\true{\pi} \subseteq \true{\pi \lor (\lambda \land \Obdot)}$, $opt(\true{\varphi}) \subseteq \true{\pi \lor (\lambda \land \Obdot)}$
				\item[] Therefore, by equation~\eqref{eq:ObligDef}, $\Model, s \models \oblig(\ (\pi \lor (\lambda \land \Obdot))|\varphi \ )$
		 \end{itemize}
\end{itemize}




\end{proof}

\section{Proof of Theorem~\ref{thm:nonestings}}\label{app:proofnonestings}

\begin{proof}
The theorem can be proved by induction on the syntax of $\chi$ (Definition~\ref{def:systemG}).

There are 5 cases to consider.

\begin{enumerate}

\item For the base case, if $\chi$ is a propositional atom, then by definition it is in UDNF. 


\item If $\chi = \neg \chi_1$ then we show that if $\chi_1$ has an equivalent formula in UDNF, so does $\chi$. Suppose $\chi_1$ = $\delta_1 \lor \dots \lor \delta_z$ for some $z \geq 1$. By Definition~\ref{def:udnf} this is in UDNF.  Then $\neg\chi_1 = \neg\delta_1 \land \dots \land\neg \delta_z$, where each $\neg \delta_i$ is of the form 
\[
\begin{array}{lcl}
\neg\delta_i & = & \neg\alpha_i \lor \neg\oblig(\varphi^i_1|\psi^i_1) \lor \dots \lor  \neg\oblig(\varphi^i_n|\psi^i_n) \\ & & \lor\ \oblig(\varphi^i_{n+1}|\psi^i_{n+1}) \lor \dots \lor \oblig(\varphi^i_{n+k}|\psi^i_{n+k})
\end{array}
\]
The result then follows by applying the distributive law: 
\[
\begin{array}{lcl}
(\gamma^1_1 \lor \dots \lor \gamma^1_{n_1}) \land (\gamma^2_1 \lor \dots \lor \gamma^2_{n_2}) \land \dots \land (\gamma^k_1 \lor \dots \lor \gamma^k_{n_k}) &  \lequiv \\
\bigvee_{1 \leq m_j \leq n_k, j \leq k} (\gamma^1_{m_1} \land \gamma^2_{m_2} \land \dots \land \gamma^k_{m_k})\\ 
\end{array}
\]

\item If $\chi = \chi_1 \lor \chi_2$  then we show that if $\chi_1$ and $\chi_2$ have an equivalent formula in UDNF, so does $\chi$. Suppose $\chi_1 \lequiv \chi'_1$ and $\chi_2 \lequiv \chi'_2$ where $\chi'_1$ and $\chi'_2$ are in UDNF (i.e. disjunctions of canonical conjunctions). Then $\chi=\chi'_1 \lor \chi'_2$ which is also a disjunction of canonical conjunctions and thus in UDNF.


\item If $\chi = \oblig(\chi_1|\chi_2)$ then we show that if $\chi_1$ and $\chi_2$ have an equivalent formula in UDNF, so does $\chi$.  In the case where both $\chi_1$ and $\chi_2$ are propositional formulae, $\chi$ is already in UDNF.  If $\chi_1$ contains a prenex formula, we can, using Lemma~\ref{lem:prenex}, assume that $\chi_1 = \pi \lor (\lambda \land \Obdot)$ or $\chi_1 = \pi \lor (\lambda \land \neg \Obdot)$ for some prenex formula $\Obdot$. In the first case, item~\ref{item:equiv:three}  of Lemma~\ref{lem:equiv} tells us how to remove the prenex outside the scope of the outer ${\mathbf O}$, in the second case we need item~\ref{item:equiv:four} to do this. Likewise, if $\chi_2$ contains a prenex, we can write $\chi_2$ either as $\pi \lor (\lambda \land \Obdot)$, or else as $\pi \lor (\lambda \land \neg\Obdot)$. In the first case, we use item~\ref{item:equiv:two} of Lemma~\ref{lem:equiv} to move the prenex outside the scope of ${\mathbf O}$, and otherwise we use item~\ref{item:equiv:one}. We can repeat this process of removing a prenex from within the scope of ${\mathbf O}$ until none are left, and finally, use the distributive law above to bring the result in normal form.





\item If $\chi = \Everywhere\chi_1$ then we show that if $\chi_1$ has an equivalent formula in UDNF, so does $\chi$. By Proposition~\ref{prop:Def} $\chi_1$ is equivalent to $\oblig(\bot|\neg \chi_1)$. An equivalent formula can be obtained by applying cases 1 to 4 above to this formula.

\end{enumerate}

\end{proof}

\section{Proof of Proposition~\ref{prop:SynDef}}\label{app:proofSynDef}
\begin{proof}

First we note the following rules of inference that are known to be sound in propositional logic namely the rules of hypothetical syllogism ({\tt HS}), biconditional introduction ({\tt BI}), subconditional elimination ({\tt SCE}), contraposition({\tt Contra}) and  substitution of provable equivalences ({\tt Subst}). 
\begin{align}
	&\mbox{If }\vdash \phi \limp \psi \mbox{ and } \vdash \psi \limp \chi \mbox{ then } \vdash \phi \limp \chi \tag{\tt HS}\\
	&\mbox{If }\vdash \phi \limp \psi \mbox{ and } \vdash \psi \limp \phi \mbox{ then } \vdash \phi \lequiv \psi \tag{\tt BI}\\
	&\mbox{If } \vdash \phi \limp (\psi \limp \chi) \mbox{ and } \vdash \psi \mbox{ then } \vdash \psi \limp \chi \tag{\tt SCE}\\
	&\mbox{If }\vdash \phi \limp \psi \mbox{ then } \vdash \neg \psi \limp \neg \phi \tag{\tt Contra}\\
	&\mbox{If } \vdash \alpha \lequiv \beta \mbox{ and } \vdash \phi \mbox{ then } \vdash \phi[\beta/\alpha] \tag{\tt Subst}
\end{align}

In {\tt Subst}, the notation $\phi[\beta/\alpha]$ means the formula exactly like $\phi$ except that instances of the formula $\alpha$ within $\phi$ are replaced by the formula $\beta$. {\tt SCE} is a special case of conditional elimination, where the conditional being eliminated is part of another conditional.
\newpage
Here then is a derivation for the formula~\eqref{eqn:SynEverywheretoO} in Proposition~\ref{prop:SynDef}. 

\[
\begin{array}{llr}
1	&	\top	&	{\tt PL} \\
2 & \Everywhere \top & {\tt 1,N} \\
3 & \Everywhere \top \limp \oblig(\top|\neg \phi) & {\tt CON}\\
4 & \oblig(\top|\neg \phi) & {\tt 2,3, MP}\\
5 & \Somewhere \neg \phi \limp ( \oblig(\top|\neg \phi) \limp \perm(\top| \neg \phi) ) & {\tt D^*}\\
6 & \Somewhere \neg \phi \limp  \perm(\top| \neg \phi)  & {\tt 5,4, SCE}\\
7 & \neg \perm(\top| \neg \phi) \limp \neg \Somewhere \neg \phi & {\tt 6,Contra}\\
8 & \oblig(\bot| \neg \phi) \limp \Everywhere \phi & {\tt 7,DfP,S5(Dual),Subst}\\
\\

9 & \phi \limp (\neg \phi \limp \bot) & {\tt PL}\\
10 & \Everywhere ( \phi \limp (\neg \phi \limp \bot)  & {\tt 9,N}\\
11 & \Everywhere \phi \limp \Everywhere (\neg \phi \limp \bot) & {\tt 10,S5(K),MP}\\
12 & \Everywhere (\neg \phi \limp \bot) \limp \oblig ( (\neg \phi \limp \bot) | \neg \phi  ) & {\tt CON}\\
13 & \oblig ( (\neg \phi \limp \bot) | \neg \phi  ) \limp ( \oblig(\neg \phi| \neg \phi) \limp \oblig(\bot|\neg \phi) ) & {\tt COK}\\
14 & \oblig(\neg \phi| \neg \phi) & {\tt Id}\\
15 & \oblig ( (\neg \phi \limp \bot) | \neg \phi  ) \limp  \oblig(\bot|\neg \phi) & {\tt 13,14, SCE}\\
16 & \Everywhere (\neg \phi \limp \bot) \limp \oblig(\bot|\neg \phi) & {\tt 12,15,HS}\\
17 & \Everywhere \phi \limp \oblig(\bot|\neg \phi) & {\tt 11,16,HS}\\
\\
18 &  \Everywhere  \phi  \lequiv \oblig(\bot|\neg \phi) & {\tt 8,17,BI}
\end{array}
\]
\end{proof}

\section{Proof of Lemma~\ref{lem:Synequiv}}\label{app:proofSynequiv}
\begin{proof}

This proof uses the propositional rules of inference {\tt HS}, {\tt BI} and {\tt Subst} given in Appendix~\ref{app:proofSynDef}.

We first derive the formula {\tt gExt} (a generalisation of ${\tt Ext}$) and ${\tt gExt^+}$ and ${\tt gExt^-}$(two special cases of ${\tt gExt}$):
\[
\begin{array}{ll}
{\tt gExt} :   &\Everywhere(\alpha \lequiv \beta) \limp (\oblig(\varphi|\pi\lor(\lambda \land \alpha)) \lequiv \oblig(\varphi|\pi\lor(\lambda \land \beta)))\\
{\tt gExt}^+ :    &\Everywhere\Obdot \limp (\oblig(\varphi|\pi\lor(\lambda \land\Obdot)) \lequiv \oblig(\varphi|\pi\lor\lambda)) \\
{\tt gExt}^- :    &\Everywhere\neg \Obdot \limp (\oblig(\varphi|\pi\lor(\lambda \land\Obdot)) \lequiv \oblig(\varphi|\pi)) 
\end{array}
\]

The derivation is as follows:
\[
\begin{array}{llr}
1 & (\alpha \lequiv \beta) \limp ((\pi\lor(\lambda \land \alpha)) \lequiv (\pi\lor(\lambda \land \beta))) & {\tt PL} \\
2 & \Everywhere  ( (\alpha \lequiv \beta) \limp ((\pi\lor(\lambda \land \alpha)) \lequiv (\pi\lor(\lambda \land \beta))) )&{\tt 1, N} \\
3 & \Everywhere   (\alpha \lequiv \beta) \limp \Everywhere ((\pi\lor(\lambda \land \alpha)) \lequiv (\pi\lor(\lambda \land \beta))) & {\tt 2, S5(K), MP}\\
4 & \Everywhere ( (\pi\lor(\lambda \land \alpha)) \lequiv (\pi\lor(\lambda \land \beta)) ) \limp \\
  &  \hspace{2cm}( \oblig(\varphi|(\pi\lor(\lambda \land \alpha))) \lequiv \oblig(\varphi|(\pi\lor(\lambda \land \beta))) ) & {\tt Ext} \\
5 & \Everywhere(\alpha \lequiv \beta) \limp (\oblig(\varphi|\pi\lor(\lambda \land \alpha)) \lequiv \oblig(\varphi|\pi\lor(\lambda \land \beta))) & {\tt 3,4,HS}\\
6 & \Obdot \limp (\Obdot \lequiv \top) & {\tt PL} \\
7 & \Everywhere (\Obdot \limp (\Obdot \lequiv \top) )& {\tt 6,N} \\
8 & \Everywhere\Obdot \limp \Everywhere(\Obdot \lequiv \top)  & {\tt 7,S5(K),MP}\\
9 & \Everywhere(\Obdot \lequiv \top) \limp (\oblig(\varphi|\pi\lor(\lambda \land\Obdot)) \lequiv \oblig(\varphi|\pi\lor(\lambda \land \top))) & {\tt gExt}\\
10 & \lambda \lequiv (\lambda \land \top) & {\tt PL}\\
11 & \Everywhere(\Obdot \lequiv \top) \limp (\oblig(\varphi|\pi\lor(\lambda \land\Obdot)) \lequiv \oblig(\varphi|\pi\lor \lambda)) & {\tt 9,10, Subst} \\
12 & \Everywhere\Obdot \limp (\oblig(\varphi|\pi\lor(\lambda \land\Obdot)) \lequiv \oblig(\varphi|\pi\lor\lambda)) & {\tt 7,11, HS}\\
\end{array}
\]
\[
\begin{array}{llr}
13 & \neg \Obdot \limp (\Obdot \lequiv \bot) & {\tt PL} \\
14 & \Everywhere (\neg \Obdot \limp (\Obdot \lequiv \bot) )& {\tt 13,N} \\
15 & \Everywhere \neg \Obdot \limp \Everywhere(\Obdot \lequiv \bot)  & {\tt 14,S5(K),MP}\\
16 & \Everywhere(\Obdot \lequiv \bot) \limp (\oblig(\varphi|\pi\lor(\lambda \land\Obdot)) \lequiv \oblig(\varphi|\pi\lor(\lambda \land \bot))) & {\tt gExt}\\
17 & \pi\lor(\lambda \land \bot) \lequiv \pi & {\tt PL}\\
18 & \Everywhere(\Obdot \lequiv \bot) \limp (\oblig(\varphi|\pi\lor(\lambda \land\Obdot)) \lequiv \oblig(\varphi|\pi)) & {\tt 16,17, Subst} \\
19 &  \Everywhere\neg \Obdot \limp (\oblig(\varphi|\pi\lor(\lambda \land\Obdot)) \lequiv \oblig(\varphi|\pi)) & {\tt 15,17, HS}\\
\end{array}
\]

%
%

Going back to the main lemma, to derive \eqref{item:equiv:five} we show that the following sequence of formulae are provably equivalent. (Note that \eqref{item:equiv:six} can be derived in a similar way, by replacing $\Obdot$ with $\neg \Obdot$.):

\[
\begin{array}{ll}
1&\Ob{\varphi}{(\pi \lor (\lambda \land \Obdot))}\\                                   
2& (\Obdot \land \Ob{\varphi}{(\pi \lor (\lambda \land \Obdot))}) \lor 
(\neg\Obdot \land \Ob{\varphi}{(\pi \lor (\lambda \land \Obdot))})\\                  
3& (\Everywhere\Obdot \land \Ob{\varphi}{(\pi \lor (\lambda \land \Obdot))}) \lor 
(\Everywhere\neg\Obdot \land \Ob{\varphi}{(\pi \lor (\lambda \land \Obdot))}) \\      
4& (\Everywhere\Obdot \land \Ob{\varphi}{(\pi \lor \lambda)}) \lor 
(\Everywhere\neg\Obdot \land \Ob{\varphi}{(\pi \lor (\lambda \land \Obdot))}) \\      
5& (\Everywhere\Obdot \land \Ob{\varphi}{(\pi \lor \lambda)}) \lor 
(\Everywhere\neg\Obdot \land \Ob{\varphi}{\pi}) \\                                    
6& (\Obdot \land \Ob{\varphi}{(\pi \lor \lambda)}) \lor 
(\neg\Obdot \land \Ob{\varphi}{\pi})\\
\end{array}
\]

${\tt 2} \lequiv {\tt 1}$  by substituting $\Ob{\varphi}{(\pi \lor (\lambda \land \Obdot))}$ for $\alpha$ and $\Obdot$ for $\beta$ in the propositional tautology $\alpha \lequiv (\beta \land \alpha) \lor (\neg \beta \land \alpha)$

${\tt 3} \lequiv {\tt 2}$, given that $\vdash \Obdot \limp \Everywhere\Obdot$ ({\tt Abs}) and $\vdash \Everywhere\Obdot \limp \Obdot$({\tt T}) by applying {\tt BI} to obtain $\vdash \Obdot \lequiv \Everywhere\Obdot$, and then {\tt Subst}.

${\tt 4} \lequiv {\tt 3}$, given the propositional tautology $(\alpha \limp (\beta \land \gamma)) \lequiv ((\alpha \land \beta) \lequiv (\alpha \land \gamma))$ and ${\tt Ext}^+$, by successive applications of {\tt Subst}. 

${\tt 5} \lequiv {\tt 4}$, in a similar way, given $(\alpha \limp (\beta \land \gamma)) \lequiv ((\alpha \land \beta) \lequiv (\alpha \land \gamma))$  and ${\tt Ext}^-$, by successive applications of {\tt Subst}.

${\tt 6} \lequiv {\tt 5}$, by given that $\vdash \Obdot \lequiv \Everywhere\Obdot$ and $\vdash \neg \Obdot \lequiv \Everywhere\neg\Obdot$ from {\tt Abs}, {\tt T} and {\tt BI}, by {\tt Subst}.

Next we also derive the formula {\tt gCOK} (a generalisation of ${\tt COK}$) and ${\tt gCOK^+}$ and ${\tt gCOK^-}$(two special cases of ${\tt gCOK}$):
\[
\begin{array}{ll}
{\tt gCOK} :    &\Everywhere   (\alpha \lequiv \beta) \limp (\oblig(\pi\lor(\lambda \land \alpha)|\psi) \limp \oblig(\pi\lor(\lambda \land \beta)|\psi))\\
{\tt gCOK}^+ :  &\Everywhere\Obdot \limp (\oblig(\pi\lor(\lambda \land \Obdot)|\psi) \limp \oblig(\pi\lor\lambda|\psi)) \\
{\tt gCOK}^- :  &\Everywhere\neg \Obdot \limp   (\oblig(\pi\lor(\lambda \land \Obdot)|\psi) \limp \oblig(\pi|\psi))
\end{array}
\]

The derivation is as follows:
\[
\begin{array}{llr}
1 & (\alpha \lequiv \beta) \limp ( \pi\lor(\lambda \land \alpha) \lequiv \pi\lor(\lambda \land \beta) ) & {\tt PL} \\
2 & \Everywhere  ( (\alpha \lequiv \beta) \limp ( \pi\lor(\lambda \land \alpha) \lequiv \pi\lor(\lambda \land \beta) ) )&{\tt 1, N} \\
3 & \Everywhere   (\alpha \lequiv \beta) \limp \Everywhere (\pi\lor(\lambda \land \alpha) \lequiv \pi\lor(\lambda \land \beta)) & {\tt 2, S5(K), MP}\\
4 & \Everywhere (\pi\lor(\lambda \land \alpha) \lequiv \pi\lor(\lambda \land \beta)) \limp & \\
   & \hspace{2cm} \oblig ( \pi\lor(\lambda \land \alpha) \lequiv \pi\lor(\lambda \land \beta) | \psi ) & {\tt CON}\\
5 & \Everywhere   (\alpha \lequiv \beta) \limp \oblig (\pi\lor(\lambda \land \alpha) \lequiv \pi\lor(\lambda \land \beta)|\psi) & {\tt 3,4.HS}\\
6 & \oblig(\pi\lor(\lambda \land \alpha) \lequiv \pi\lor(\lambda \land \beta)|\psi) \limp &\\
   & \hspace{2cm} (\oblig(\pi\lor(\lambda \land \alpha)|\psi) \limp \oblig(\pi\lor(\lambda \land \beta)|\psi)) & {\tt COK}\\
7 & \Everywhere   (\alpha \lequiv \beta) \limp (\oblig(\pi\lor(\lambda \land \alpha)|\psi) \limp \oblig(\pi\lor(\lambda \land \beta)|\psi)) & {\tt 5,6,HS}
\\
8 & \Obdot \limp (\Obdot \lequiv \top) & {\tt PL} \\
9 & \Everywhere (\Obdot \limp (\Obdot \lequiv \top) )& {\tt 8,N} \\
10 & \Everywhere\Obdot \limp \Everywhere(\Obdot \lequiv \top)  & {\tt 9,S5(K),MP}\\
11 & \Everywhere(\Obdot \lequiv \top) \limp  (\oblig(\pi\lor(\lambda \land \Obdot)|\psi) \limp \oblig(\pi\lor(\lambda \land \top)|\psi)) & {\tt gCOK}\\
12 & \lambda \lequiv (\lambda \land \top) & {\tt PL}\\
13 & \Everywhere(\Obdot \lequiv \top) \limp  (\oblig(\pi\lor(\lambda \land \Obdot)|\psi) \limp \oblig(\pi\lor\lambda|\psi)) & {\tt 11,12, Subst} \\
14 &  \Everywhere\Obdot \limp (\oblig(\pi\lor(\lambda \land \Obdot)|\psi) \limp \oblig(\pi\lor\lambda|\psi)) & {\tt 10,12, HS}\\
\end{array}
\]
\[
\begin{array}{llr}
15 & \neg \Obdot \limp (\Obdot \lequiv \bot) & {\tt PL} \\
16 & \Everywhere (\neg \Obdot \limp (\Obdot \lequiv \bot) )& {\tt 15,N} \\
17 & \Everywhere \neg \Obdot \limp \Everywhere(\Obdot \lequiv \bot)  & {\tt 16,S5(K),MP}\\
18 & \Everywhere(\Obdot \lequiv \bot) \limp  (\oblig(\pi\lor(\lambda \land \Obdot)|\psi) \limp \oblig(\pi\lor(\lambda \land \bot)|\psi)) & {\tt gCOK}\\
19 & \pi\lor(\lambda \land \bot) \lequiv \pi & {\tt PL}\\
20 & \Everywhere(\Obdot \lequiv \bot) \limp  (\oblig(\pi\lor(\lambda \land \Obdot)|\psi) \limp \oblig(\pi|\psi)) & {\tt 18,19, Subst} \\
21 &  \Everywhere\neg \Obdot \limp   (\oblig(\pi\lor(\lambda \land \Obdot)|\psi) \limp \oblig(\pi|\psi))& {\tt 15,18, HS}\\
\end{array}
\]

    
Going back to the main lemma again, to derive \eqref{item:equiv:seven} we show that the following sequence of formulae are provably equivalent. (Note that \eqref{item:equiv:eight} can be derived in a similar way, by replacing $\Obdot$ with $\neg \Obdot$.):

\[
\begin{array}{ll}
1&\Ob{\pi \lor (\lambda \land \Obdot)}{\psi}\\								                   
2& (\Obdot \land \Ob{\pi \lor (\lambda \land \Obdot)}{\psi}) \lor 
(\neg\Obdot \land \Ob{\pi \lor (\lambda \land \Obdot)}{\psi})\\                  
3& (\Everywhere\Obdot \land \Ob{\pi \lor (\lambda \land \Obdot)}{\varphi}) \lor 
(\Everywhere\neg\Obdot \land \Ob{\pi \lor (\lambda \land \Obdot)}{\varphi})\\    
4&  (\Everywhere\Obdot \land \Ob{\pi \lor \lambda}{\psi}) \lor 
(\Everywhere\neg\Obdot \land \Ob{\pi \lor (\lambda \land \Obdot)}{\psi})\\       
5&(\Everywhere\Obdot \land \Ob{\pi \lor \lambda}{\psi}) \lor 
(\Everywhere\neg\Obdot \land \Ob{\pi}{\psi})\\		                               
6&(\Obdot \land \Ob{\pi \lor \lambda}{\psi}) \lor 
(\neg\Obdot \land \Ob{\pi}{\psi})
\end{array}
\]
   
${\tt 2} \lequiv {\tt 1}$  by substituting $\Ob{\varphi}{(\pi \lor (\lambda \land \Obdot))}$ for $\alpha$ and $\Obdot$ for $\beta$ in the propositional tautology $\alpha \lequiv (\beta \land \alpha) \lor (\neg \beta \land \alpha)$

${\tt 3} \lequiv {\tt 2}$, given that $\vdash \Obdot \limp \Everywhere\Obdot$ ({\tt Abs}) and $\vdash \Everywhere\Obdot \limp \Obdot$({\tt T}) by applying {\tt BI} to obtain $\vdash \Obdot \lequiv \Everywhere\Obdot$, and then {\tt Subst}.

${\tt 4} \lequiv {\tt 3}$, given the propositional tautology $(\alpha \limp (\beta \land \gamma)) \lequiv ((\alpha \land \beta) \lequiv (\alpha \land \gamma))$ and ${\tt COK}^+$, by successive applications of {\tt Subst}. 

${\tt 5} \lequiv {\tt 4}$, in a similar way, given $(\alpha \limp (\beta \land \gamma)) \lequiv ((\alpha \land \beta) \lequiv (\alpha \land \gamma))$  and ${\tt COK}^-$, by successive applications of {\tt Subst}.

${\tt 6} \lequiv {\tt 5}$, by given that $\vdash \Obdot \lequiv \Everywhere\Obdot$ and $\vdash \neg \Obdot \lequiv \Everywhere\neg\Obdot$ from {\tt Abs}, {\tt T} and {\tt BI}, by {\tt Subst}.

\end{proof}

\end{document}